\documentclass[12pt]{article}
\usepackage{times}
\usepackage{amsmath}
\usepackage{amsthm}
\usepackage{graphicx}
\usepackage{float}
\usepackage{commath}
\usepackage[sort]{natbib}
\setcitestyle{square}
\usepackage{hyperref}
\setcitestyle{numbers}
\setcitestyle{comma}
\usepackage[ruled,vlined]{algorithm2e}
\include{pythonlisting}
\usepackage{fancyhdr}
\usepackage[parfill]{parskip}
\newtheorem{theorem}{Theorem}
\newtheorem{corollary}{Corollary}[theorem]

\newtheorem{definition}[theorem]{Definition}
\newtheorem{proposition}[theorem]{Proposition}
\usepackage{setspace}
\usepackage[dvipsnames]{xcolor}
\allowdisplaybreaks
\let\oldsection\section
\renewcommand\section{\clearpage\oldsection}

\begin{document}

\begin{titlepage}
   \begin{center}
       \vspace*{0cm}
       \large{\text{B.Comp. Dissertation}}
 
       \vspace{0.5cm}
        \textbf{Fair Multi-party Machine Learning - a game theoretic approach}
 
       \vspace{1.5cm}
 
       \textbf{By} \\
       \vspace{0.5cm}
        Chen Zhiliang
       
        \vspace{4cm}
      Department of Computer Science\\
       School of Computing\\
       National University of Singapore\\
       2019/2020
 
       \vspace{0.8cm}
   \end{center}

\end{titlepage}

\begin{titlepage}
   \begin{center}
       \vspace*{0cm}
 
       \large{\text{B.Comp. Dissertation}}
 
       \vspace{0.5cm}
        \textbf{Fair Multi-party Machine Learning - a game theoretic approach}
 
       \vspace{1.5cm}
 
        \textbf{By} \\
       \vspace{0.5cm}
        Chen Zhiliang
       
        \vspace{3cm}
      Department of Computer Science\\
       School of Computing\\
       National University of Singapore\\
       2019/2020
 
       \vspace{0.8cm}
   \end{center}
   \vspace{2cm}
   Project no.:	H148310\\
   Advisor: Prof Bryan Kian Hsiang Low\\
    \\
   Deliverables:\\
   \hspace{2mm}Report - 1 volume
\end{titlepage}
\begin{center}
\large{Abstract}
\end{center}

\hspace{4mm} High performance machine learning models have become highly dependent on the availability of large quantity and quality of training data. To achieve this, various central agencies such as the government have suggested for different data providers to pool their data together to learn a unified predictive model, which performs better. However, these providers are usually profit-driven and would only agree to participate in the data sharing process if the process is deemed both profitable and fair for themselves. Due to the lack of existing literature, it is unclear whether a fair and stable outcome is possible in such data sharing processes. Hence, we wish to investigate the outcomes surrounding these scenarios and study if data providers would even agree to collaborate in the first place. Tapping on cooperative game concepts in Game Theory, we introduce the data sharing process between a group of agents as a new class of cooperative games with modified definition of stability and fairness. Using these new definitions, we then theoretically study the optimal and suboptimal outcomes of such data sharing processes and their sensitivity to perturbation. Through experiments, we present intuitive insights regarding theoretical results analysed in this paper and discuss various ways in which data can be valued reasonably.
\\ \\ \\ \\
\textbf{Subject Descriptors:} \\
\hspace{4mm} Game Theory \\
\hspace{4mm} Mechanism Design \\
\hspace{4mm} Machine Learning
\\ \\ \\ \\
Keywords: Multi-party Machine Learning, Cooperative Games, Game Theory, Shapley Value, Characteristic Function

\begin{center}
\large{Acknowledgement}
\end{center}
I would like to thank my advisor, Prof. Bryan Low, for guiding me throughout this project. He showed me numerous traits a good researcher should have and never hesitated to point me in the right direction whenever I appeared unsure of myself.

I would also like to thank my parents, who have always supported me in my endeavours. Lastly, I thank Ong Min for offering me many words of encouragement and being my emotional pillar of support throughout the years. 

\tableofcontents
\medskip

\section{Introduction}
In recent years, there has been an increasing amount of work done in the field of collaborative machine learning. As massive amount of data is held by different data providers, it becomes worthwhile for them to work together, possibly combining these datasets to learn a unified predictive model. Each data provider can then ideally make use of this model, which arguably performs better than any model created from small individual dataset. In fact, the importance of pooling data together from multiple sources has been underscored by various public initiatives such as the \textit{Ocean Protocol} framework (created jointly by Pricewaterhousecoopers Singapore and Singapore-based startup DEX).
From the point of view of a central agency, such as the government, data sharing is desirable because citizens, businesses and society stand to benefit from better predictive models in general.

Most work on collaborative machine learning in a multi-agent setting then focuses on parallelizing the process of learning a predictive model from decentralized data sources. However, there has been little to no work focused on evaluating the relative contribution of agents and investigating the outcomes surrounding such collaborative processes. In particular, we wish to investigate if we can measure each agent's relative contribution reasonably and find a sufficiently fair reward to allocate to each player so that they are satisfied with the collaborative process. \begin{figure}[h]
\includegraphics[scale=0.6]{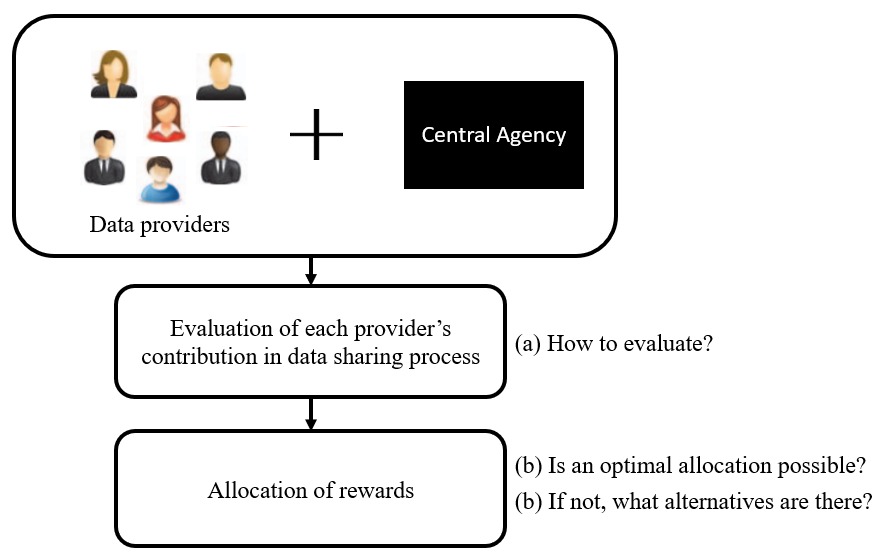}
\centering
\caption{Overview of data sharing process and this paper's focus}
\end{figure}\\
The first challenge arises in terms of evaluating the contribution of each agent in a multi-party machine learning process - when data is jointly pooled from multiple sources to create a unified model, what is the relative contribution of each agent in the collaborative effort? Answering this question is paramount because the reward each agent walks away from collaboration heavily depends on his relative contribution.

The next challenge arises after one has established the relative contribution of each agent in the collaborative process - given this contribution measure, how can we reward the participating agents appropriately after collaboration (the concept of 'reward' here refers to the value of the resulting model that each agent receives after collaborating)? While the reward given to an agent should clearly commensurate his contribution (fair), it should also be attractive enough to ensure agents are incentivised to continue collaborating (stable). We need to investigate if such fair and stable allocation of reward is even possible given the data contributed by each agent, and if not, what possible alternatives are available.
To contextualise the above challenges, we introduce the following motivating example.
\subsection{A motivating example behind fair multi-party machine learning}
Artificial intelligence has shown great promise in healthcare for many applications such as tumour detection and the diagnosis of eye diseases. However, much of its performance is dependent on the availability of sufficient high quality data.
\begin{figure}[h]
\includegraphics[scale=0.5]{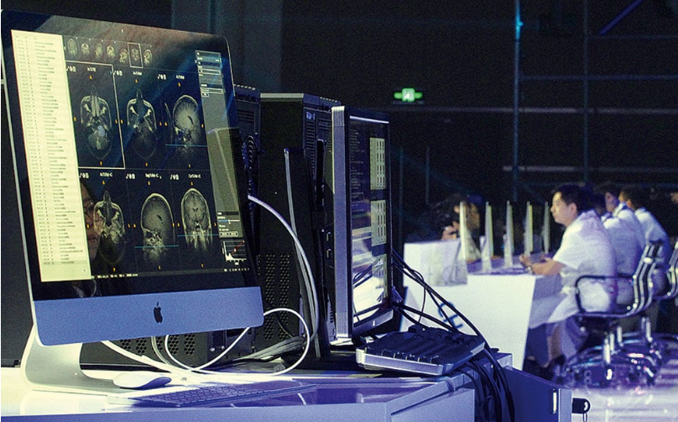}
\centering
\caption{BioMind, developed by AI Research Centre for Neurological Disorders and Capital Medical University, beat human experts in a brain tumour detection competition in June 2018}
\end{figure}
Assume a government healthcare agency wants hospitals to use patients' MRI \footnote{Magnetic resonance imaging} scans to create a model for automatic tumour detection. However, data held by individual hospital may be insufficient to produce a good model. From the point of view of the government, it is desirable for all the hospitals to pool their data to create a unified model - this implies that patients enjoy better predictive results. 
However in reality, hospitals are profit-driven and will only agree to participate if the following two criteria are met: first, a hospital wants the reward gained from the collaborative effort to be higher than any possible reward derived when it chooses to work alone or with a smaller group of hospitals. Second, a hospital expects the reward received to be fair and proportional to its relative contribution at the end of the collaboration. The differing level of reward here refers to the predictive model (of varying performance) that is given back to the hospitals.

The first criterion stems directly from the profit-driven aim of hospitals - a better predictive model implies that a hospital can make more profits and thus clearly it prefers to gain more profits from collaboration than from other actions (one may find this similar to the idea of opportunity cost). On the other hand, the second criteria implies that hospitals holding high quality data, worried that smaller hospitals with relatively inferior data may become freeloader in the collaborative process, expects the reward given to be fair.

As a result, the central agency needs to find a sound way to measure the contribution of each hospital in the collaborative process and decide the value of model that should be given to each hospital after collaborating. In particular, it is important to investigate what defines an optimal outcome, if an optimal outcome is even possible and if not, what alternatives are possible. This motivates our paper.

\subsection{Our contribution and research focus}
This paper taps heavily on concepts related to cooperative games in Game Theory.
Our research focus in this paper lies strongly in the theoretical representation of a multi-party machine learning process as a new class of cooperative game with new definitions of stability and fairness. After which, we analyse various optimal and suboptimal outcomes in such games and show how some of these outcomes can be derived. These outcomes provide insights into how data providers behave in a multi-party machine learning process and allow us to theoretically investigate if the collaboration will be successful and give desirable results.
\\ \\
In Section 3, we provide a brief overview of cooperative games in Game Theory and show that conventional cooperative games are unable to account for some unique subtleties present in data sharing and other similar processes. In Section 4 and 5, we modify conventional assumptions in cooperative games to account for this and introduce a new class of cooperative game with the production of non-rivalrous goods; we also redefine the concept of stability and fairness in such games by tapping on the conventional concept of \textit{Shapley value} in Game Theory. In Section 6, we introduce a computationally efficient way to check if an outcome of a modified game is stable and investigate how an outcome can be optimal. In Section 7, We also analyse various suboptimal outcomes in a data sharing process. Lastly, in Section 8, we perform sensitivity analysis to analyse how sensitive an optimal outcome is towards marginal changes in the data sharing process.

We also perform a series of experiments in Section 9 using two different model valuation measures and demonstrate some results from previous sections.

\section{Related Work}
Literature review reveals that while the study of fairness in the domain of data sharing and collaborative learning has become popular in recent years, most work has focused solely on evaluating the contribution of data providers in a data sharing process or treating the process as non-cooperative. In addition, none of existing works analyses how a fair and stable outcome can be derived after evaluating the data providers' contribution.
\\ \\
The use of \textit{Shapley value} \cite{9}, a classic concept stemming from cooperative game, as an evaluation measure of one's contribution in data sharing processes \cite{2,3} has become increasingly popular in recent years. However, these works focus on alleviating the computational complexities in calculating the \textit{Shapley value} with respect to agents' data; while the usage of \textit{Shapley value} in these works is similar to ours, they do not comment on how one can derive an optimal reward payoff to participating data providers after finding the \textit{shapley value}, which our work does. Moreover, these works also fail to point out that a data sharing process contains subtleties which deviate from conventional cooperative games, which our paper examines in Section 4.
\\ \\
Furthermore, works such as \cite{5,6} attempt to formulate the data sharing process as a non-cooperative game, where each participating player contributes data to create a unified predictive model. However, such non-cooperative game rewards each participating players with the same model at the end of the game, which violates our aim to endow players with a reward commensurate to his contribution. Moreover, non-cooperative games fails to account for the possibility of group cooperation between players, which not only is more realistic in real life, but also examines the idea of fair division of reward at the end of the game.
\\ \\
Some existing literature \cite{2,7} also remarked that participating players could receive the same model and be renumerated with varying amount of money after data sharing. However, it is uncertain how to determine the "exchange rate" between data contribution and money and some surveys \cite{8} mentioned that people feel that data and money are not exchangeable. On the contrary, our work focuses on distributing a subset of the learnt model (with value commensurating the player's contribution to this learnt model) back to participating players directly; this is desirable because as long as we defined the "value" of a learnt model appropriately using some theoretical measures, the players' reward (in the form of a learnt model) can be directly pegged to this "value".

\section{Background on cooperative games}
In Game Theory, cooperative games provide a framework for economists and mathematicians to study how a group of self-interested players chooses to behave when they are allowed to cooperate with each other to generate some resources of value. Despite being self-interested, these players may still choose to cooperate because there may be positive externalities created when larger groups are formed. However, it is not always the case that all players choose to cooperate in one single large group; depending on the resources created and the amount of positive externalities generated, it is entirely possible that we observe multiple smaller coalition of players forming. In fact, this is what we observe in real life: business cooperation usually exists between only a few corporations. However, our paper considers it desirable for all players to cooperate and show that under certain optimality conditions, they will choose to do so.
\\ \\
Much analysis in cooperative games then focuses on what coalitions will form and how the resources generated in such coalitions can be divided amongst the participating players (termed as \textit{outcomes} of the game). In the following sections, we will formally define such cooperative games \cite{10} and investigate how such definitions can be interpreted in the context of a multi-party machine learning process. Subsequently, we identify how conventional cooperative game concepts fail to represent the data sharing process adequately and suggest modifications.
\subsection{Defining a cooperative game with a characteristic function}
Assume that $N$ represents a set of $n$ players. Define $v: 2^n \xrightarrow{} \mathcal{R}$ as a \textit{characteristic function}. A \textit{characteristic function} maps any coalition $C \subseteq N$, or a subset of players, to a real number which represents the value that is generated when this coalition of player chooses to work together.
\\ \\
In the context of multi-party machine learning, we can imagine that when players choose to work together (in the form of data sharing) to create a model from the pooled data, this model represents the value created by a coalition of players; the characteristic function $v(.)$ can be certain statistical or information-centric metric used to indicate the value of such models. Furthermore, one can even use the model's performance over a test data set as an indicator of its value.

In fact, various analyses in this paper make no assumptions of the exact characteristic function used in a data sharing process. However, towards the end of the paper, we justify the use of certain functions over others. As we will see from experiments in section 9, there are many different interpretations of characteristic functions and they lead to different outcomes.

\subsubsection{Outcomes of cooperative games}
An \textit{outcome} of a cooperative game with $n$ players and a given characteristic function $v(.)$ is characterised entirely by the following two parts: \begin{enumerate}
    \item a partition of players into different coalitions, called a \textit{coalition structure}; and
    \item a \textit{payoff vector} $<x_1,...,x_n>$, which distributes the value of each coalition  among its members
    $\sum_{i\in C} x_i \leq v(C)$ for all coalition $C$ formed in the \textit{coalition structure}.
 \end{enumerate}
The first part tells us which coalitions will be formed; the second part tells us that for each coalition formed in the first part, how value created in this coalition is divided amongst the players in it. In particular, we notice that second part implies that the sum of payoff distributed to members of a coalition cannot exceed the value generated by this coalition; for example, when a group of firms work together to generate some profits, the payment given to all the firms in total cannot exceed the profit generated in the first place.

For any given cooperative game, there are many possible outcomes. Fundamental research in cooperative game theory focuses on finding outcomes with certain desirable or logical properties. In the following section, we review some desirable properties in which an outcome of a conventional cooperative game can have, assuming that all players choose to cooperate together. While the definition of these properties will be modified later on in our paper, the intuitive meaning behind these properties is still relevant.

\subsubsection{Stability and Core}
Assuming that the grand coalition $N$ forms and all participating players choose to cooperate, a stable outcome ensures that any player does not have incentive to leave this grand coalition. In the context of multi-party machine learning, this property is desirable because it encourages players to share data in a larger coalition to create a unified model instead of breaking off into smaller groups.
\begin{definition}
Let $<x_1,x_2,...,x_n>$ be a payoff vector to players {1,2,...,n}. Then this payoff vector is \textbf{stable} if and only if $\sum_{i \in C} x_i \geq v(C)$ for all $C \subseteq N$.
\end{definition}
Intuitively, an outcome is stable if the reward received by each player at the end is such that no subset of players can simultaneously break off into smaller sub-coalitions which generate higher rewards than the sum of payment each player earns from the current reward. As the name suggests, one can interpret stability as some kind of equilibrium for the players, where there is no incentive for anyone to take deviating actions.

In conventional cooperative games, the set of all \textbf{stable} solution payoffs is called the $\textbf{Core}$. Notice that deriving the stable set of solution consists of finding the feasible region satisfying $2^n$ constraints (by \textbf{Definition 1}). Hence, the \textbf{Core} does not necessarily exist.
\subsubsection{Shapley value}
On the other hand, when we divide the value created in the grand coalition $N$ to the participating players in a payoff vector, we also wish to capture the notion of fairness in this outcome. While the notion is "fairness" is not captured in any one single definition, it can be represented reasonably by a few different properties.

In cooperative games, the \textit{\textbf{shapley value}} derives a payoff vector $<\phi_1,\phi_2,\dots,\phi_n>$ based on each player's marginal contribution to all possible sub-coalitions in the following equation, given that $v(.)$ is the characteristic function: \begin{equation}
    \phi_{i}=\sum _{S\subseteq N\setminus \{i\}}{\frac {|S|!\;(N-|S|-1)!}{N!}}(v(S\cup \{i\})-v(S))
\end{equation}
This payoff vector satisfies the following properties which capture the notion of "fairness" intuitively:
\begin{enumerate}
    \item \textbf{Symmetry}; if $v(C \cup i)-v(C) = v(C \cup j)-v(C)$ for all $C \subseteq N$, then $\phi_i = \phi_j$. That is, if player $i$ and player $j$ has the same marginal contribution to every possible coalition of players, then $\phi_i=\phi_j$
    \item \textbf{Null player}; if $v(C \cup i) = v(C)$ for all $C \subseteq N$, then $\phi_i=0$. That is, if player $i$ has zero marginal contribution to any subcoalition, then $\phi_i=0$
    \item \textbf{Deservedness} if $v(C \cup i) \leq v(C \cup j)$ for all $C \subseteq N$, then $\phi_i \leq \phi_j$. That is, if player $j$ has a larger marginal contribution than player $i$ for all possible subcoalitions, then $\phi_i \leq \phi_j$.
    \item \textbf{Efficiency}; $\sum^{n}_{i}\phi_i=v(N)$.
\end{enumerate}
The last property is important because it allows one to interpret the \textit{shapley value} as a measure of relative contribution with respect to the value of resources created by the grand coalition. In particular, it is not difficult to see that the \textit{shapley value} gives us one particular \textit{outcome} (as defined in Section 3.1.1) of a cooperative game by allocating the value created by the grand coalitions to a player with the properties above.
\\ \\
Unfortunately, the allocation given by the \textit{Shapley value} may not be \textit{stable} in conventional cooperative games (since a stable solution may not even exist). In later sections, we will observe the recurring theme where fairness and stability may not be achievable together.
\subsubsection{Limitations in context of data sharing}
Even though we are tempted to model the data sharing process as a cooperative game and regard the combined learnt model as the value generated by a group of players, we observe that under conventional game theory, the ideal outcomes (based on conventional concepts of stability and fairness) of such games are less than desirable and do not promote sharing of data. For example, imagine if two agents A and B whose datasets are valued in the following manner:
\begin{equation}
\begin{split}
 & v(A) = 1
\\
 & v(B) = 1
 \\
 & v(A \cup B) = 2
\end{split}
\end{equation}
Based on conventional outcomes studied in cooperative game theory, if the players choose to collaborate, the only stable solution payoff is ${x_A,x_B} = {1,1}$. Furthermore, the shapley value gives us $\phi_A,\phi_B = 1,1$ as well. But clearly, since players are producing predictive models, which can be duplicated for free, we should be rewarding both players with the payoff ${x_A,x_B} = {2,2}$ (notice we did not create resources from the thin air, but merely duplicated them).
\\ \\
We first try to understand why this happens. In example (2), the \textit{shapley value} explains the contribution of each player well: it seems correct that both players have contributed to $v(A \cup B)$ equally. However, we observe that, unlike conventional cooperative games, predictive models created from the data sharing process can be duplicated in part or entirely for negligible cost. As such, we do not have to restrict our total payoff to a coalition of players such that it sums to the value of one data model. In fact, in the extreme case, we can distribute the same entire learnt model to each player, regardless of how much they contributed to it (of course, this payoff is not fair). As such, the \textit{shapley value} should merely give us an idea of what the relative contribution of each player is, and not dictate the final division of reward to each player.

In the following section, we will modify conventional assumptions surrounding cooperative games to suit a multi-party machine learning data-sharing process and analyse what can be ideal outcomes in such modified games. In particular, we observe that the model created by coalitions in a data-sharing process can be categorised as a \textit{non-rivalrous good}, where players can enjoy part of it without diminishing its availability to other players. Furthermore, we use the \textit{Shapley value} merely as a measure of contribution between the players and not as an indicator of the reward that should be allocated to each of them.
\section{Deviating from conventional cooperative games}
\subsection{Overview}

We wish to capture the notion that resources created are \textit{non-rivalrous} in our modified game. In fact, each player can receive as high as the total value created by the group. However, the central agency mediating the collaborative process is able to control the level of reward given to any player. To do this, we need to redefine what a \textit{payoff vector} is.\begin{definition}
In a modified game, a \textbf{payoff vector} $<x_1,x_2,...,x_n>$ to a coalition structure is such that for all coalitions $C$ belonging to the coalition structure formed, $x_i \leq v(C)$ for any player $i \in C$.
\end{definition}
Compared to the definition of \textit{payoff vector} in Section 3.1.1, this new definition restricts the payment to a coalition of players such that each player cannot receive more than the value created in this coalition. Notice this is different from conventional cooperative games, which dictate that the sum of payments received by all the players cannot exceed the value created in a coalition.

The data sharing process can now be represented entirely by this modified game with a new definition of \textit{payoff vector} in an outcome. However, changing this single definition leads to a myriad of effects on the concept of stability, fairness and outcome, which will be covered in the subsequent sections.

\subsection{Alternative definition of stability}
Under this modified game, the original definition of stability (\textbf{Definition 1}) no longer ensures that a payoff vector is stable. For example, for a three-player game with the following characteristic function defining the value of models created in various coalitions: \begin{equation}
\begin{split}
 & v(A) = 1, v(B) = 1, v(C) = 1
\\
 & v(A \cup B) = 2, v(A \cup C) = 2, v(B \cup C) = 2
 \\
 & v(A \cup B \cup C) = 3
\end{split}
\end{equation}
A payoff vector $<x_A,x_B,x_C> = <1.5,1.5,1.5>$ would have been stable according to \textbf{Definition 1} because it is easily verifiable the sum of payoffs to any subset of players exceed the value created by them. However, in a modified game where resources created by a coalition can be duplicated for free, player $A$ and player $B$, for example, can choose to leave the grand coalition and work together to both earn a payoff of higher than 1.5 (Since they can pool their data together in subcoalition $A \cup B$ to potentially create a model of value 2 and earn a payoff $<x_A, x_B> =<2,2>$). Thus, the original definition of stability fails for our modified game.

We redefine, in an intuitive way, the definition of stability in a modified game: \begin{definition}
In a modified game where the grand coalition $N$ of $n$ players forms, a payoff vector $<x_1,x_2,...,x_n>$ is \textbf{stable} if and only if for any subcoalitions $C \subseteq N$, there exists a player $k \in C$ such that $x_k \geq v(C)$
\end{definition}
An intuitive way to understand this definition is that players in the grand coalition are only satisfied with their current reward if there are no subcoalitions where a subset of players can choose to form privately and simultaneously earn a higher reward. As such, considering that the maximum payoff one can earn in any subcoalition is the value created by the subcoalition itself, there must be at least one player in every possible subcoalition $C$ who is contented with his current payoff under the grand coalition (i.e $x_k \geq v(C)$), preventing that subcoalition of players from deviating privately (since the refusal of a single player is sufficient to prevent a group of players from deviating together).
\subsection{Definition of proportionality}
In Section 3, we introduced how the \textit{Shapley value} derives a payoff vector $\phi(N) = <\phi_1, \phi_2, ... , \phi_n>$ $= <x_1,x_2,\dots,x_n>$ (assuming all $n$ players work together in coalition $N$) which dictates how the value generated from the grand coalition should be divided amongst its members to maintain some fair properties. In fact, the \textit{shapley value} indicates the level of contribution of each member with respect to the model created of value $v(N)$. 

However, in a modified game, \textbf{Definition 2} places bounds on the payoff to individual players instead on the sum of payoffs. We can scale the vector $\phi(N)$ by a positive constant $\alpha$ such that the resulting payoff vector $<x_1,\dots,x_n> = \alpha\phi(N)$ preserves the ratio of contribution between each player. Here, we introduce the concept of proportionality to categorise any payoff vector satisfying the following property:
\begin{definition}
A payoff vector $<x_1,x_2,\dots,x_n>$ is \textbf{proportional} for a given contribution measure vector $\mathcal{C}\in\mathcal{R}^n$ if and only if  there exists a positive constant $\alpha$ such that $<x_1,x_2,\dots,x_n> = \alpha  \mathcal{C}$.
\end{definition}

In the rest of the paper, we will assume the \textit{shapley value} as our contribution measure because a payoff vector proportional to the \textit{shapley value} preserves its desirable properties covered in Section 3.1.3.
\section{Desirable properties of outcomes in modified cooperative game}
Recall that an outcome of a cooperative game is defined by the resulting coalition structure formed and payoff vector defined over the coalitions in the coalition structure (Section 3.1.1). Some outcomes satisfy certain desirable properties in data sharing. Here, we summarise these properties.

Given a modified game with a group of $n$ players, a characteristic function $v(.)$ used to measure the value of predictive model created by the data sharing process and the subsequent \textit{shapley value} $<\phi_1, \phi_2, ... , \phi_n>$ derived as the contribution measure, we want an outcome consisting of a coalition structure $\mathcal{S}$ and payoff vector $<x_1,x_2,\dots,x_n>$ to have the following properties:
\begin{enumerate}
    \item \textbf{Formation of grand coalition} coalition structure $\mathcal{S}$ consists of a single grand coalition with $n$ players. That is, all players choose to cooperate in the data sharing process.
    \item \textbf{Stability (Definition 3)} All players continue to collaborate in a large group.
    \item \textbf{Fairness} \begin{enumerate}
        \item \textbf{Proportionality (Definition 4)} Players receive reward proportional to his relative contribution.
        \item \textbf{Null Player} If $\phi_i = 0$, then $x_i=0$; a player with no contribution gets zero reward. 
        \item \textbf{Symmetry} If $\phi_i = \phi_j$, then $x_i=x_j$; two players with equal contribution receives equal reward.
        \item \textbf{Order Preserving} If $\phi_i > \phi_j$, then $x_i>x_j$; if player $i$ has a lower contribution than player $j$, then player $i$'s reward is lower than that of player $j$.
    \end{enumerate}
\end{enumerate}
The rest of this paper investigates whether an outcome satisfying all the above properties (regarded as optimal) exists and other suboptimal outcomes.
\subsection{Assumptions of modified cooperative game}
Before we formalise the definition of an optimal outcome, we make the following basic assumptions regarding our modified game:
\begin{enumerate}
    \item We assume that characteristic function $v(.)$ used to value predictive models is \textbf{monotonic}. That is, if $C \subseteq C'$, then $v(C) \leq v(C')$; this assumption implies that the value (as valued by the characteristic function) of the model generated by the largest group of players $N$ is also the largest. This assumption implies that larger coalition generates more value than smaller ones.
    \item We use the \textit{\textbf{shapley value}}, as defined in Section 3.1.3, as the relative contribution measure $\mathcal{C}$ to gauge a player's contribution in the data sharing process. As mentioned, the \textbf{\textit{shapley value}} has many desirable properties.
\end{enumerate}
With these assumptions, we introduce a general algorithm to determine the reward allocated to each player in a modified cooperative game. This algorithm will be used by a central agency who receives the data from all participating players.
\begin{algorithm}[h]
\KwIn{Datasets of $n$ players $D_1, D_2,\dots,D_n$, characteristic function $v: 2^n \mapsto \mathcal{R}$ }
\KwOut{rewards $x_1,x_2,\dots,x_n$}

\nl $\phi \gets$ shapley($v(.), D_1,D_2,\dots,D_n$);

\nl \If{$getOptimalOutcome(\phi,v) \neq None$}{$x_1,x_2,\dots,x_n \gets getOptimalOutcome(\phi,v)$\;
return $x_1,x_2,\dots,x_n$}

\nl \Else{$x_1,x_2,\dots,x_n \gets getSuboptimalOutcome(\phi,v)$\;
return $x_1,x_2,\dots,x_n$}
\caption{{\bf Data-sharing and reward division} \label{Algorithm}}

\end{algorithm}

In the following sections, we analyse how functions $getOptimalOutcome$ and $getSuboptimalOutcome$ derive optimal and suboptimal outcomes. The central agency can use these results to allocate rewards for the participating players.
\section{Optimal outcome}

To contextualise an optimal outcome in our modified game to a description of a real-life data-sharing process, imagine $n$ data providers come together to create a predictive model from the pooled data. Then using a monotonic characteristic function $v(.)$ to measure the value of the predictive model, we are able to derive the relative contributions of the players with the \textit{Shapley value}. Subsequently, an optimal outcomes corresponds to a reward payoff for the players such that they have no incentives to break away from the grand coalition and the reward also commensurates each player's relative contribution. In this section, we investigate how we can find this optimal outcome efficiently.
\\ \\
We consider an outcome to a modified game \textit{optimal} if they satisfy all properties in the Section 5. Hence, the task of finding an optimal outcome becomes one of finding the feasible region of the division of rewards to the players $<x_1,x_2,\dots,x_n>$ such that all desirable properties are satisfied. However, this is computationally inefficient because the definition of stability alone implies that we need to check $2^n$ constraints. Fortunately, we show, in the following proposition, that the set of stable solution payoffs can be found by checking only $n$ constraints, given that the characteristic function $v(.)$ is monotonic.

\begin{proposition} \textbf{Stable Solution Set}
Given a monotonic valuation function $v(.)$, let $<x_1,x_2, ... ,x_n>$ be a payoff vector to $n$ agents, arranged in ascending order. Then this payoff vector is stable if and only if $\forall i \in \{1,2,\dots,n\}$,\begin{equation}
    x_i \geq v(\bigcup_{j\leq i} j)
\end{equation} 
\end{proposition}
\begin{proof}
if $\forall i, x_i \geq v(\bigcup_{j\leq i} j)$, then for any sub coalition $C^{'}$ containing a subset of agents, the agent in $C^{'}$ with the largest index (we refer to this index as $k$) would be such that $x_k \geq v(\bigcup_{j\leq k} j)$. Since $k$ is the largest index in $C^{,}$, it follows that $C^{'} \subseteq \bigcup_{j\leq k} j$ and because of our assumption that $v(.)$ is monotonic, we have $x_k \geq v(\bigcup_{j\leq k} j) \geq v(C^{'})$. Thus by definition of stability, the payoff vector $<x_1,...,x_n>$ is stable. 

Conversely, if the payoff vector is stable, then by definition, for all index $k$ and for every subcoalition $\bigcup_{j\leq k} j$, there exists at least one agent with payoff larger than $v(\bigcup_{j\leq k} j)$. Since, $x_k$ is the highest payoff received in this subcoalition (because we rearranged the payoff index in ascending order), we have $x_k \geq v(\bigcup_{j\leq k} j)$ for all $k$.
\end{proof}
This theorem implies that to ensure players are rewarded such that they have no incentive to break off and form smaller private coalitions, their reward must be larger than the value of model created by the union of all players with less or equal contribution as him. This is a simple definition but also an extremely important one because it allows us to define the space of rewards such that no players are incentivised to break off from the grand coalition. 
\begin{figure}[h]
\includegraphics[scale=0.7]{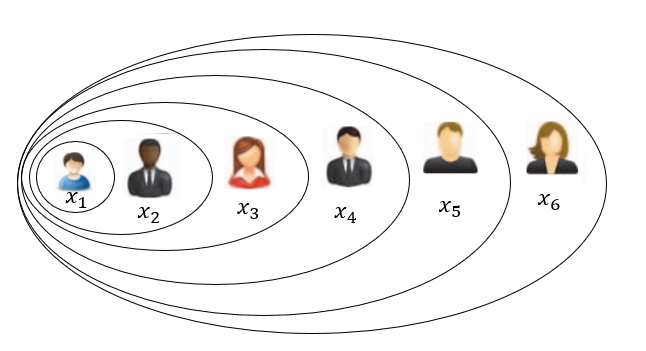}
\centering
\caption{Visualisation of \textbf{Proposition 5}. Lining up the players in ascending contribution, the payoff $x_i$ to each player must be larger than the value created by the union of previous players to ensure stability.}
\end{figure}

With the above proposition on finding the stable solution set, it is then natural for us refine this \textbf{stable set} with properties related to \textbf{fairness} to define an \textbf{optimal outcome}:
\begin{theorem}
 \textbf{Optimal outcome}
Given a monotonic valuation function $v(.)$ with the value created by grand coalition equals $v(N)$, let $<x_1,x_2, ... ,x_n>$ be a solution payoff to $n$ agents, arranged in ascending order and $<\phi_1, \phi_2,...,\phi_n>$ be the relative contribution of the agents of the same index. Then this solution payoff is \textbf{stable and proportional} if and only if $\forall i \in \{1,2,\dots,n\}$, 
\begin{equation}
    x_i = \phi_i\frac{v(N)}{\phi_n} \geq v(\bigcup_{j\leq i} j)
\end{equation}
An outcome with such a solution payoff is then said to be \textbf{optimal}.
\end{theorem}
\begin{proof}
In a stable solution, the player with the highest index receives a payoff $x_n \geq v(\bigcup_{j\leq n} j) = v(N)$ by definition of stability. This implies that the agent receiving the highest payoff must receive $v(N)$. By observation, for the solution to be proportional, this agent also has the highest contribution $\phi_n$ and thus every other agent $i$ has its contribution $\phi_i$ scaled by a factor of $\frac{v(N)}{\phi_n}$. The inequality follows naturally in order to abide to the condition of stability.

Now, since the \textit{shapley value} is used as the relative contribution measure, the \textbf{proportional and stable} solution payoff inherits other \textbf{fairness} properties: \textbf{Null Player, Symmetry, Order Preserving}. As such, an outcome with such a solution payoff satisfies all properties in Section 5 and is \textbf{optimal}.
\end{proof}
\textbf{Theorem 6} gives us an optimal outcome of a multi-agent data sharing process formulated as a modified cooperative game. Intuitively, this theorem tells us that an optimal reward allocation to each player is such that it is proportional to the contribution measure and also large enough to prevent them from breaking away from the grand coalition.

The central agency mediating the process could use the theorem to check if an optimal outcome is possible and if so, use it to allocate rewards to the participating players appropriately. It is also not difficult to see that the theorem implies that an optimal outcome, if it exists, is unique.

\section{Suboptimal Outcomes}
The optimal outcome for a group of $n$ players may not exist given a particular characteristic function and contribution measure indicated by the \textit{Shapley value}. That is, there may be no outcomes satisfying \textbf{Theorem 6}. This implies that in a data sharing process, the central mediator may not find an allocation of the final reward such that all players regard it as fair and stable at the same time.

Nevertheless, there are still suboptimal outcomes which are of interest to us; in particular, since having a \textbf{stable and proportional solution payoff} is necessary and sufficient for an optimal outcome, we can restrict our analysis on suboptimal outcomes with the following solution payoffs:
\begin{itemize}
    \item a solution payoff which is stable but not proportional
    \item a solution payoff which is proportional but not stable
\end{itemize}
Analysis of these two classes of suboptimal outcomes gives us an idea on how to derive a suboptimal outcome which is amenable to the players when an optimal one is not achievable. As a mediator promoting the data sharing process between players, a government agency can derive a suboptimal outcome with some but not all desirable properties and perhaps still convince the players to continue collaborating in the data-sharing process, especially if the suboptimal outcome does not violate stability or proportionality too much.

\subsection{Stable but not Proportional outcome}
This outcome implies that each player will receive a reward that ensures he is not incentivised by larger rewards elsewhere to leave the grand coalition. However, some players may receive a reward that is not proportional to his relative contribution and if players can tolerate some disproportionality, then this outcome is still acceptable.

Let $<\phi_1, ... , \phi_n>$ be the shapley value of the agents arranged in ascending order. Proportionality holds between the payoffs $x_i,x_j$ of any two players $i$ and $j$ if $\frac{\phi_i}{\phi_j} = \frac{x_i}{x_j}$. Hence, if proportionality cannot be achieved in a stable outcome, we need a metric to measure the degree of proportionality violation in a given payoff vector and select the stable outcome which violates this metric the least.

The concept of proportionality deviation is widely studied in proportional representation systems such as allocation of parliament seats\cite{11}.  For example, we can measure the sum of pairwise absolute proportionality violations between the reward received by all players: \begin{equation}
    Deviation_{sum}= \sum_{i,j \in N} \norm{\frac{\phi_i}{\phi_j}-\frac{x_i}{x_j}}_p
\end{equation}
This certainly is not the only available deviation measure. One alternative can be the max individual deviation from proportionality. This measure is useful if we assume agents are able to tolerate some deviation from proportionality (a certain level of unfairness): \begin{equation}
    Deviation_{max}= max\left( \norm{\frac{\phi_i}{\phi_j}-\frac{x_i}{x_j}}_p\right)_{i,j}
\end{equation}
Hence, finding a stable solution with the lowest deviation measure can be formulated as the following optimisation problem:
\begin{subequations}
\begin{alignat}{2}
&\! \min_{x_1,...,x_n}        &\qquad& DeviationMeasure\label{eq:optProb}\\
&\text{subject to} &      & x_i \geq v\left (\bigcup_{j\leq i} j\right),\label{eq:constraint1}\\
&                  &      & \text{additional constraints}.\label{eq:constraint2}
\end{alignat}
\end{subequations}
Objective function (8a) represents a deviation measure such as (6) or (7); constraint (8b) represents the solution payoff's stability constraint, since we need to outcome to be stable; constraints (8c) represents additional constraints which may be required if one wishes to enforce certain additional \textbf{Fairness} properties (\textbf{Symmetry, Order Preserving, Null Player}), since without which, a non-proportional solution may not necessarily guarantee these properties. 
\\ \\
We acknowledge that finding an appropriate deviation measure is a non-trivial task as there are drawbacks to different deviation measures (\cite{11} mentions 19 different kinds of proportionality deviation indices!); for example, measure (5) may assign large disproportional reward to a single player while measure (6) yields non-unique solutions. As such, the central agency in the data sharing process needs to design a deviation measure appropriate to reconcile a reasonable non-proportional outcome. However, empirical results at the end of the paper show that $Deviation_{sum}$ gives reasonable outcomes.

\subsection{Proportional but not stable outcome}
A proportional but not stable outcome implies that while the reward given to each player is fair, some players may be incentivised to break away from the grand coalition for better rewards. To tackle this realistically in real life,  the central agency must first find out how "far away" the proportional solution is from each player's lower bounds of stability. 
\begin{definition} \textbf{$\epsilon$-Stability} Let $<x_1, ..., x_n>$ be the solution payoff to the agents, then this solution payoff is $\epsilon$\textbf{-stable} if for all $C \subseteq N$, there exists $i \in C$ such that $x_i \geq (v(C) - \epsilon)$
\end{definition}
Notice \textbf{Definition 7} is identical to the stability definition apart from the a deduction of $\epsilon$ in the lower bound. If $\epsilon > 0$, then we can understand $\epsilon$ as the penalty each player pays for deviating from the grand coalition to form a sub-coalition. As such, when a solution is $\epsilon$\textbf{-stable}, if players have to pay a penalty larger than $\epsilon$ to break off from the grand coalition, then these players will remain in the grand coalition because they will be unable to reap higher rewards elsewhere after paying the penalty to leave.

Instead of being viewed as a penalty, $\epsilon$ can also be viewed as additional compensation to incentivise players to stay in the grand coalition. The central agency could try to give additional benefits equivalent to $\epsilon$ to ensure players remain in the grand coalition (of course, this requires agencies and players to discuss what denomination these benefits come in). The following corollary allows one to find this $\epsilon$.
\begin{corollary}
Let $<x_1,x_2,...,x_n>$ be a proportional but not stable payoff vector. Define
\[
  d_i =
  \begin{cases}
                    0 & \text{if $x_i \geq v(\bigcup_{j\leq i} j)$} \\
                    v(\bigcup_{j\leq i} j)-x_i & \text{otherwise} \\
  \end{cases}
\]
Then the payoff vector is $\epsilon$-stable if $\epsilon \geq \max\limits_{i} (d_i)$.
\end{corollary}
\begin{proof} 
Given any solution payoff $<x_1,...,x_n>$, since for any subcoalition $C$, there exists $i$ such that $ C \subseteq \bigcup_{j\leq i} j$ and thus $ v(C) \leq v(\bigcup_{j\leq i} j)$ by the monotonicity of $v(.)$. It follows directly that for every agent $k$ and every subcoalition $C$ with $k$ being the highest indexed agent belonging to this subcoalition, $\epsilon \geq \max\limits_{i} (d_i) \geq v(C) - x_k \xrightarrow{} x_k \leq (v(C) - \epsilon)$.
\end{proof}

\section{Sensitivity Analysis} Given that an optimal outcome is decided for a modified game where players have received a learnt model of certain value, there can be many elements that influence the process retrospectively. Sensitivity analysis allows us to investigate how sensitive the optimal outcome is towards such changes. In particular, we wish to investigate if an outcome will remain optimal with such changes. Even if such perturbations have not occurred, the central agency can perform sensitivity analysis pre-emptively to investigate how robust the current optimal outcome is towards possible marginal changes in the future.
\begin{figure}[h]
\includegraphics[scale=0.7]{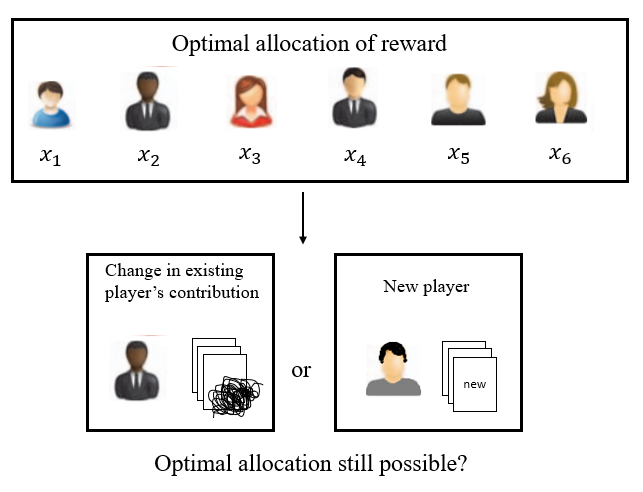}
\centering
\caption{Sensitivity analysis of two possible scenarios}
\end{figure}

For example, one may realise retrospectively that a segment of a player's data is unusable due to privacy concerns and needs to be removed from the model created (leading to a marginal decrease in contribution of that player); on the other hand, a player could have his relative contribution increased marginally by a constant amount because he was able to offer some non-data related help in the data sharing process (e.g offering GPU processing power).
\\
In addition, new players may enter the collaboration process. While we can reinstate the entire game with an additional player, the process to calculate the \textit{Shapley value} of all players again using $v(.)$ requires us to pool the data in $2^n$ subsets, making the process time consuming. As such, if an expert can estimate the marginal contribution of the new player, we show in the following section that we can estimate if an optimal outcome is still possible without complex recalculations.
\subsection{Inclusion of an new agent}
First, we estimate the new agent's contribution relative to existing agents using some expert opinion (making educated guess based on some features surrounding the data set or the data provider). Second, we estimate how much each coalition's value changes due to the inclusion of the new agent. For the first estimate, we let the index of the new agent be $new$ and introduce $\phi_{new}$ to the existing contribution vector (whilst preserving the ascending order of the vector): $<\phi_1,\phi_2,...,\phi_n> $ becomes $ <\phi_1,\phi_2,...,\phi_{new},...,\phi_n>$. For the second estimate, it is computationally expensive to derive the value of each subcoalition containing new agent $new$; for ease of computation, we assume that for every subcoalition $C$ in the original group of agents, $v'(C \cup new) \approx v(C) + \phi_{new}$. This implies the marginal increase in a value of a subcoalition by agent $new$ is the same and leads to an overall relative contribution  $\phi_{new}$, as calculated by the shapley value formula.
\\ \\
Let the new grand coalition, $N'=(N \hspace{1mm} \cup \hspace{1mm} new)$, contain $n+1$ players. Hence, $v'(N') = v'(N \hspace{1mm} \cup \hspace{1mm} new) \approx v(N) + \phi_{new}$. Also assume the current outcome containing the payoff vector $<x_1,...,x_n>$ is stable, proportional and thus optimal. We use $w$ to indicate index of agents with lower contribution than $new$ and $s$ to indicate index of agents with higher contribution than $new$.
\begin{theorem}
A stable and proportional solution, and thus an optimal outcome, is still possible with a new player with marginal contribution $\phi_{new}$ to all existing subcoalitions if and only if $\phi_{new}$ satisfies the following conditions:

\begin{align}
\label{eqn:eqlabel}
\begin{split}
 & \forall s, \hspace{2mm} \phi_{s} \frac{v(N)+\phi_{new}}{\phi_{max}} \geq  v\left(\bigcup_{j\leq s} j \middle\backslash new \right) + \phi_{new},
\\
 & \phi_{new}\frac{v(N)+\phi_{new}}{\phi_{max}} \geq v(\bigcup_{j < new}j) + \phi_{new}
\end{split}
\end{align}
\end{theorem}
\begin{proof}
\textbf{Theorem 8} stems directly from optimality conditions in \textbf{Theorem 6}. For the first inequality in (9), in a game with a new player with $\phi_{new}$, an optimal outcome is possible if and only if
\begin{equation} \label{eq1}
\begin{split}
\forall s, & \hspace{2mm} x_s = \phi_s \frac{v'(N')}{\phi_{max}} \geq v'(\bigcup_{j \leq s}j) \hspace{53mm} \text{(Theorem 6)} \\
 & \iff \phi_s \frac{v(N) + \phi_{new}}{\phi_{max}} \geq v'(\bigcup_{j \leq s}j) \hspace{40mm} \text{(Estimation of $v'$ on the left)} \\ & \iff \phi_s \frac{v(N) + \phi_{new}}{\phi_{max}} \geq v\left(\bigcup_{j\leq s} j \middle\backslash new \right) + \phi_{new} \hspace{10mm} \text{(Estimation of $v'$ on the right)}
\end{split}
\end{equation}
The proof is identical for the second inequality in (9).
\end{proof}
Notice that all terms except $\phi_{new}$ in the above theorem is already computed in the original game. Hence, all we have to do is check if $\phi_{new}$ satisfies the inequalities. Furthermore, it is interesting to note that if an outcome is already optimal, the entry of a new player does not influence the optimality conditions on existing players with lower relative contributions than the new player.
\subsection{Perturbation of contribution of one particular player}
Next, we analyse the case where all subcoalitions containing a particular player $i$ is changed by a constant value. As mentioned, this may occur when the central agency retrospectively needs to perturb the value of one player's data due to privacy concerns or if we retrospectively realise that a segment of a player's data is corrupted. Again, we make an estimation of the new value generated by all coalitions containing player $i$ in the following manner: let $v'(C) = v(C)+\delta$ for all coalition $C$ containing agent $i$. Then this implies a linear shift in player $i$'s \textit{shapley value}: $\phi_{i}^{new} = \phi_{i} + \delta$ (by definition of \textit{shapley value}). Also assume that $\delta$ is such that the order of relative contribution of agents does not change and the current outcome containing the payoff vector $<x_1, ..., x_n>$ is already optimal.

For a particular agent of index $k$, let $d_k = \frac{\phi_k}{\phi_{max}}v(N) - v(\bigcup_{j\leq k} j) \geq  0$, which always holds because the current soluton is optimal. Again, we use $w$ to indicate index of agents with lower contribution than $i$ and $s$ to indicate index of agents with higher contribution than $i$. \begin{theorem}
A stable and proportional solution, and thus an optimal outcome, is still possible under $\phi_i^{new} = \phi_i + \delta$ for player $i$ if and only if $\delta$ satisfies the following conditions: 
\begin{equation}
\begin{split}
 & \forall w, \hspace{2mm} \delta \geq - \frac{\phi_{max}}{\phi_w}d_w , 
\\
 &  \forall s, \hspace{2mm} \delta \leq \frac{d_s \phi_{max}}{\phi_{max}-\phi_s},
 \\
 & \delta^2 + (\phi_i + v(N) - \phi_{max})\delta +d_i\phi_{max} \geq 0, 
\end{split}
\end{equation}
\end{theorem}

\begin{proof}
We prove all three inequalities in \textbf{Theorem 9} in order. An optimal outcome is possible under new conditions where player's $i$'s contribution is perturbed by $\delta$ if for the first inequality: 
\begin{equation} \label{eq2}
\begin{split}
\forall w, & \hspace{2mm} x_w = \phi_w \frac{v(N')}{\phi_{max}} \geq v'(\bigcup_{j \leq w}j) \hspace{53mm} \text{(Theorem 6)} 
\\ & 
\iff \phi_w \frac{v(N) + \delta}{\phi_{max}} \geq v(\bigcup_{j \leq w}j) \hspace{40mm} \text{(Estimation of $v'$)} \\ &
 \iff \phi_w \frac{v(N)}{\phi_{max}}-v(\bigcup_{j \leq w}j) \geq -\frac{\phi_w}{\phi_{max}}\delta \hspace{10mm} \\ &
 \iff d_w \geq -\frac{\phi_w}{\phi_{max}}\delta \hspace{52mm} \text{(By definition of $d_w$)} \\ &
 \iff \delta \geq -\frac{\phi_{max}}{\phi_w} d_w 
\end{split}
\end{equation}
For the second inequality:
\begin{equation}
\begin{split}
    \forall s, & \hspace{2mm} x_s = \phi_s \frac{v(N')}{\phi_{max}} \geq v'(\bigcup_{j \leq s}j) \hspace{63mm} \text{(Theorem 6)} 
\\ & 
\iff \phi_s \frac{v(N) + \delta}{\phi_{max}} \geq v(\bigcup_{j \leq s}j) + \delta \hspace{40mm} \text{(Estimation of $v'$)} 
\\ & 
\iff \phi_s \frac{v(N)}{\phi_{max}}-v(\bigcup_{j \leq s}j) \geq -\frac{\phi_s}{\phi_{max}}\delta + \delta
\\ &
\iff d_s \geq (-\frac{\phi_s}{\phi_{max}} + 1)\delta \hspace{49mm} \text{(By definition of $d_s$)} \\ &
 \iff \delta \leq \frac{\phi_{max}}{\phi_{max}-\phi_s} d_s 
\end{split}
\end{equation}
Lastly, for the third equality:
\begin{equation}
\begin{split}
    & x_i = (\phi_i + \delta) \frac{v(N')}{\phi_{max}} \geq v'(\bigcup_{j \leq i}j) \hspace{63mm} \text{(Theorem 6)} 
\\ & \iff (\phi_i+\delta) \frac{v(N) + \delta}{\phi_{max}} \geq v(\bigcup_{j \leq i}j) + \delta \hspace{38mm} \text{(Estimation of $v'$)}
\\ & \iff \phi_i \frac{v(N)}{\phi_{max}}-v(\bigcup_{j \leq i}j) +\frac{\phi_i\delta + \delta v(N) + \delta^2}{\phi_{max}} \geq \delta
\\ & \iff d_i + \frac{\phi_i\delta + \delta v(N) + \delta^2}{\phi_{max}} \geq \delta  \hspace{43mm} \text{(By definition of $d_i$)}
\\ & \iff
d_i \phi_{max} + \phi_i \delta + \delta v(N) + \delta^2 \geq \delta \phi_{max}
\\ & \iff 
\delta^2 + (\phi_i + v(N) - \phi_{max})\delta +d_i\phi_{max} \geq 0
\end{split}
\end{equation}
\end{proof}
To make sense of the inequalities presented in \textbf{Theorem 9}. We first notice that in the trivial case where $\delta=0$, each inequality holds because $d_k \geq 0$ for any $k$. 

Furthermore, when $\delta > 0$ (when the contribution of a player is increased), then the first inequality is true because the right expression in the last line of (12) is always negative; the third inequality also holds true because the quadratic inequality on the last line of (14) has positive coefficients. This implies that we only have to check whether the second inequality holds true when $\delta > 0$.

Lastly, when $\delta < 0$ (when the contribution of a player is decreased), the second inequality holds true; that is, we do not have to worry about infeasibility of payoffs to players with lower contribution than the perturbed player. For sanity check, we also notice that there exists some $\delta < 0$ such that the first inequality holds true in the last line of (12) since the right term is negative. Lastly, it is easily verifiable that the third inequality, the quadratic inequality $\delta^2 + (\phi_i + v(N) - \phi_{max})\delta +d_i\phi_{max} \geq 0$ holds true for some values of $\delta <0$ because the coefficients are all non-negative.

\section{Experiments}

Let $D_i$ be the dataset held by player $i$. We emulate a multi-party machine learning process through the use of one-dimensional synthetic datasets $D_1,D_2,\dots,D_7$ held by 7 data providers and demonstrate the properties and outcomes mentioned in our paper. In particular, we demonstrate the following: \begin{enumerate}
    \item Desirable properties of the \textit{Shapley value} in measuring the contribution of each player's dataset based on two different characteristic functions - \textit{Fisher Information} and \textit{Mutual Information}.
    \item Optimality of outcomes induced by these characteristic functions (i.e is there a fair and stable way to reward the players?)
    \item Deriving a suboptimal outcome when an optimal outcome is not achievable.
    \item Will an optimal outcome still be achievable if we perturb the contribution of one player marginally?
\end{enumerate}
In addition to demonstrating the above, we also offer some intuitive and useful insights to various observations made in the experiments with regards to the data sharing process. We also explain why we select \textit{Fisher Information} and \textit{Mutual Information} as characteristic functions over other alternatives.
\subsection{Introducing two different characteristic function}
As mentioned previously, the characteristic function $v(.)$ used to measure the value of predictive model created should be reasonable and context dependent. One immediate idea that one has with regards to evaluating models is to evaluate its performance with a test data set. While this seems promising in evaluating the power of each data provider's data set, this is not always achievable in reality because it is non-trivial to gather a test data set is truly representative of inputs that future predictions will be based on. Furthermore, the central agency may not have access to such data. Instead, it may be more favourable to rely on certain statistical measures to evaluate the level of uncertainty associated with the learnt predictive models.
\\ \\
We introduce two different characteristic functions, \textit{Fisher Information} and \textit{Mutual Information}, which are both deeply rooted in probability. First, we give a brief introduction of these functions and discuss their appropriateness in valuing predictive models.
\subsubsection{Fisher Information}
\textit{Fisher Information} \cite{12} measures the amount of information that an observed data carries about the parameter $\theta$ used to model the data. In particular, the Fisher Information of a set of data points gives us an idea about the variance of an unbiased parameter estimate generated from these data points, directly allowing us evaluate the uncertainty surround the predictive model (which comprises of these parameters). Traditionally, statisticians have used Fisher Information to design experiments to ensure estimated parameters lie within a small confidence interval.
\begin{definition}
Let $f(X; \theta)$ be the probability density function for random variable(s) $X$ conditional on a true parameter $\theta$. Then provided $f(X;\theta)$ is twice differentiable and under some regularity conditions, the \textbf{Fisher Information} is defined as $ \mathcal{I}(\theta;X)=-\operatorname {E} \left[\left.{\frac {\partial ^{2}}{\partial \theta ^{2}}}\log f(X;\theta )\right|\theta \right]$
\end{definition}
The more interesting result follows from the derivation of the \textit{Cramer Rao bound} after $\mathcal{I}(\theta)$ is derived, which lower bounds the precision that we can estimate parameter $\theta$.
\begin{definition}
Let $\hat{\theta}$ be any unbiased estimator of $\theta$ derived from observed data $X$. Then \textbf{Cramer Rao bound} states that $Var(\hat{\theta}) \geq \frac{1}{\mathcal{I}(\hat{\theta};X)}$, where the equality holds when the estimated parameter is efficient.
\end{definition}
This bound implies that when estimating parameters, maximising the \textit{Fisher Information} is equivalent minimising the smallest attainable variance about an estimated parameter. In fact, in many estimators, the equality holds.
\\
In the context of a multi-party machine learning process, we can derive the \textit{Fisher Information} of a set of data to derive the smallest attainable variance of the estimated parameter; a higher Fisher Information for a group of players' data implies possibly less uncertainty. This allows us to define, for any coalition $C$ of players with combined dataset $D_C$, $v_1(C) =  \mathcal{I}(\theta;D_C)$. 

In Linear Regression of the form $y = x^T\theta + \epsilon$, where $\epsilon \sim \mathcal{N}(0, \sigma^2 I)$, for a dataset $D_C$ containing $n$ observed input and output values $(X_i;Y_i)$ with noise $\sigma_{i}$:
\begin{equation}
    v_1(C) = \mathcal{I}(\theta;D_C) = \sum_{i=1}^n \frac{E(X_{i}^TX_{i})}{\sigma_{i}}
\end{equation}
We notice that $v_1(D_1) + v_1(D_2) = v_1(D_1 \cup D_2)$ for any two datasets $D_1$ and $D_2$. This implies that $v_1(.)$ is an additive and monotonic characteristic function. In fact, the variance of the estimated parameter is equals the reciprocal of $v_1(.)$ in Linear Regression.

\subsubsection{Mutual Information}
\textit{Mutual Information} measures the reduction in uncertainty regarding model parameters when one has access to a set of data $D$. Even though it shares a similar idea with \textit{Fisher Information} in that both gauge the value of data based on uncertainty of parameters, \textit{Mutual Information} accounts for the prior uncertainty regarding the parameters. One may prefer to use \textit{Mutual Information} if one has some idea what the "inherent uncertainty" associated with the model parameters is.
\begin{definition}
The \textbf{Mutual Information} between two continuous variables $(A,B)$ is defined as $M.I(A,B)=
{\displaystyle\int _{\mathcal {Y}}\int _{\mathcal {X}}{p_{(X,Y)}(x,y)\log {\left({\frac {p_{(X,Y)}(x,y)}{p_{X}(x)\,p_{Y}(y)}}\right)}}\;dx\,dy,}$
\end{definition}

In a Bayesian Linear Regression setting, let a dataset $D_C$ contain $n$ output values $(X_i,Y_i)$ (represented by $X \in \mathcal{R}^{n\times p}$, where $p$ is the number of parameters) with noise $\delta_i$ associated with each observation $i$. Assume $\delta \in \mathcal{R}^{n \times n}$ is a diagonal matrix with diagonals = $\delta_i$ and $\Sigma_{prior}$ is the prior covariance of parameters. Then we have:
\begin{equation}
    v_2(C) = M.I(\theta,D_C) = 0.5 \text{log}|\Sigma_{prior}||\Sigma_{prior}^{-1}+X^T \sigma^{-1}X|
\end{equation}
We notice that $v_2(.)$, unlike $v_1(.)$, is not additive but still monotonic. 
\subsection{Experiment results}
We first synthesise one-dimensional data held by seven different players, shown in Figure 5. The legend also shows the distribution of inputs $X$ of the data held by each player, along with the output noise $\delta$ associated with each player. Notice player 1 holds a single datapoint at the origin and player 3 and 4 hold datapoints at similar input spaces.

\begin{figure}[h]
\includegraphics[scale=0.7]{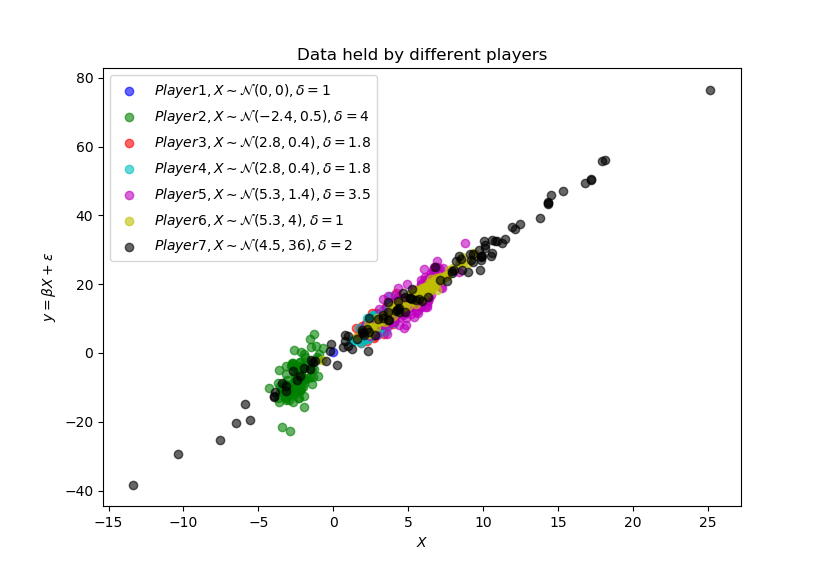}
\centering
\caption{Data held by different players}
\end{figure}

\subsubsection{Shapley value}
We then derive the \textit{shapley value} $\phi_i$, our contribution metric, for each player $i$, based on $v_1(.)$ and $v_2(.)$ in Table 1. Because $v_1(.)$ and $v_2(.)$ hold different interpretations towards the value of data, we observe that even though both valuations preserve the players' order of contribution, $v_1$ gives much larger disparity between players' contribution as compared to $v_2$. This can be inferred directly from the functions' definitions. As observed from Equation (15), $v_1(.)$ regards the value of 3 datapoints (sampled from similar locations and with equal noise level) to have 3 times the value of 1 datapoint; this leads to large ratios between the players' contribution. On the other hand, $v_2(.)$'s formulation in Equation (16) gives diminishing returns when more data is included; this leads to smaller ratio between the players' contribution.
\begin{center}
 \begin{tabular}{||c c c c c c c c||} 
 \hline
   & $\phi_1$ & $\phi_2$ & $\phi_3$ & $\phi_4$ & $\phi_5$ & $\phi_6$ & $\phi_7$   \\ [0.5ex] 
 \hline
 $v_1$ & 0 & 162 & 500 & 501 & 1007 & 3003 & 3868  \\
 \hline
 $v_2$ & 0.002 & 1.404 & 1.613 & 1.617 & 1.884 & 2.314 & 2.327  \\ 
 \hline
\end{tabular}
  \\[1.5mm] 
  Table 1: \textit{Shapley value} based on two different valuation function $v_1$ and $v_2$
\end{center}
Despite the differences in interpretation and contribution values obtained from both characteristic functions, we notice some similarities when they derive the \textit{shapley value} $\phi$: \begin{enumerate}
    \item Player 1 is a \textbf{null player} with a single data point $(0,0)$, which is generally useless in estimating the model parameter. Hence, both characteristic functions calculate $\phi_1$ as close to zero.
    \item Player 3 and Player 4 hold almost identical data (sampled from the same distribution) and thus $\phi_3$ and $\phi_4$ are almost identical for both characteristic functions.
    \item The order of player's contribution is the same for both characteristic functions.
\end{enumerate}
As such, despite differences in $v_1(.)$ and $v_2(.)$, they both offer some intuitive and desirable properties when deriving the contribution measure (these properties generally relate to how appropriately fair we are evaluating players' contribution). When generating an appropriate level of rewards to the players based on this contribution measure later, we will see that such desirable properties will be passed onto the outcomes.

\subsubsection{Characteristic functions affect optimality of outcomes}
\textbf{Definition 4} and \textbf{Proposition 5} defines the space of proportional and stable solution respectively. Using these definitions, we show the lower bound for stability for each player (recall that stability is defined by lower bounds on the reward given to each player in \textbf{Proposition 5}) and the proportional solution for $v_1(.)$ and $v_2(.)$ respectively in Figure 6. Moreover, \textbf{Theorem 6} directly allows us to check if an optimal outcome is possible based on whether the proportional solution is stable; hence, Figure 6 also showcases whether the outcome is optimal for both characteristic functions. 

The choice of an appropriate characteristic function \textbf{should not} be motivated by whether said function is able to achieve an optimal outcome. The outcome, whether optimal or not, is simply the result from a function choice. The choice of a particular function to value a model should still depend on its theoretical or practical appropriateness.
\begin{figure}[h]
\includegraphics[scale=0.7]{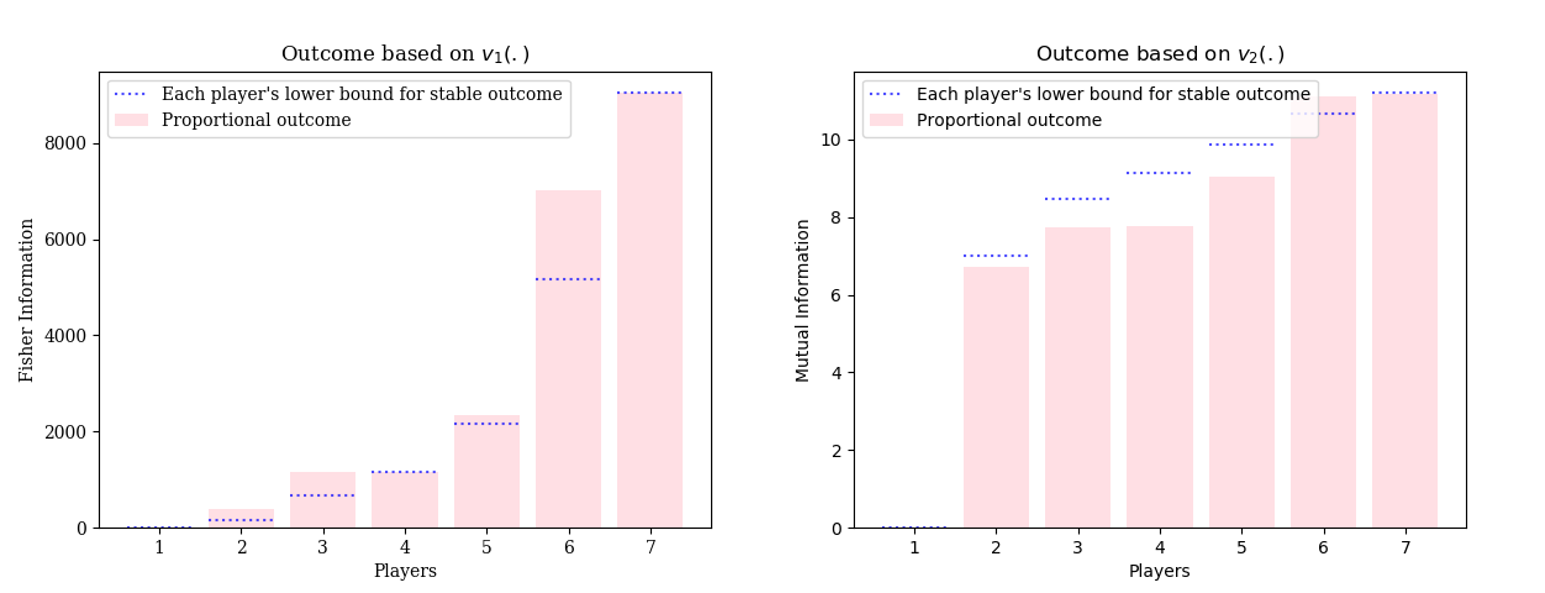}
\centering
\caption{Stablility bound and Proportional solution for $v_1$ and $v_2$}
\end{figure}

We immediately notice the following differences between the solutions induced by two different characteristic functions: \begin{enumerate}
    \item For $v_1(.)$, we see that for each player, the proportional reward given (red bars) is larger or equals to his own lower stability bound (blue dotted lines). Hence by \textbf{Theorem 6}, the payoff vector indicated by the red bars to the players is an optimal outcome; notice that player 3 and 4 enjoy equal amount of reward because they contributed data of similar value and player 1 receives no reward.
    \item For $v_2(.)$, we notice that the proportional solution payoff for Player 2,3,4 and 5 is lower than the stability lower bound defined by \textbf{Proposition 5}. Hence, the proportional solution is not stable and an optimal outcome is not achievable.
\end{enumerate}
As such, if the data providers and central agency agree to use $v_1(.)$ as the characteristic function used to evaluate the value of a model, then an optimal outcome can be achieved in the data sharing process. However, only suboptimal outcomes can be achieved with $v_2(.)$
\\ To intuitively explain why $v_2(.)$ does not allow an optimal outcome, notice that the inequality in \textbf{Theorem 6} (which checks if an optimal outcome is possible): $\phi_i\frac{v(N)}{\phi_n} \geq v(\bigcup_{j\leq i} j)$ most likely holds if the contribution of player $i$, $\phi_i$, is not much less than the combined value of the model created by players with lower or equal contribution as player $i$. In other words, if a player's data is generally useful (dictated by how much the player's data marginally increases the value of different coalitions of players) to every coalition, then his relative contribution should not be too small. In the larger picture, we need this inequality to hold for every player, suggesting that an optimal outcome is possible only when the players' data is generally equally useful to all different coalitions. Notice that $v_2(.)$, \textit{Mutual Information}, in Bayesian Linear Regression is a metric that values addition of data less when there is already a lot of data present; thus, the marginal addition of player's data to larger groups of players is generally not very useful, causing the inequality to be violated and the optimal outcome to be not achievable. On the other hand, \textit{Fisher Information} values datapoints additively without diminishing contribution, making players' data generally useful for every coalition and allowing the optimal outcome to be achievable.

\subsubsection{Settling for a suboptimal outcome}
To further demonstrate the analysis on suboptimal outcomes in Section 7, we focus on the case where the central agency, upon finding that an optimal solution cannot achieved, attempts to seek a suboptimal outcome which is still stable. This coincides with the outcome raised in Section 7.1, where players still receives a "somewhat proportional" reward (defined by a deviation measure). Since $v_2(.)$ cannot achieve an optimal outcome, we demonstrate the suboptimal outcome based on $Deviation_{sum}$ minimisation (Equation 8a) in Figure 7.

\begin{figure}[h]
\includegraphics[scale=0.5]{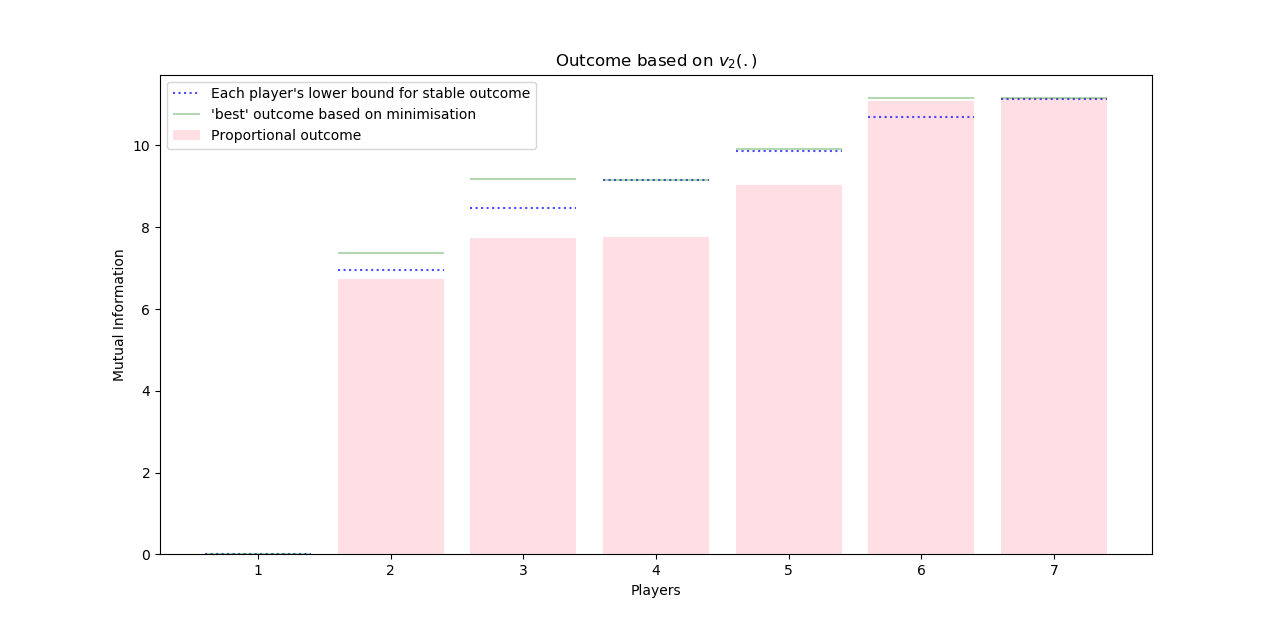}
\centering
\caption{Suboptimal outcome (green line) based on $Deviation_{sum}$ minimisation}
\end{figure}

From Figure 7, one notices that the suboptimal outcome (green line) is stable and preserves certain level of proportionality between players. For example, player 1, a null player, continues to receive zero reward, whereas player 3 and 4 still receive the same level of reward in the suboptimal outcome. However, the rewards given to other players are not proportional to their relative contribution anymore.

It is easy to see that a suboptimal outcome which compromises proportionality in return for stability will always yield higher reward for players. This is because the proportional reward allocation is not stable and thus we have to increase the reward given to achieve stability.
\subsubsection{Perturbation of contribution of one particular player}
We would also like to demonstrate how \textbf{Theorem 9} allows us to check if an optimal outcome is still achievable given that the contribution of one player (with index $i$) changes marginally. Recall that this implies the following: \begin{equation}
    \begin{split}
       &  v'(C) = v(C)+\delta \text{ for all coalition } C \text{ containing agent } i \\
       & \phi^{new}_i = \phi_i + \delta 
    \end{split}
\end{equation}
For the rest of the section, we assume that $\delta > 0$, implying that the contribution of a data provider has marginally increased (due to addition of better data or denoising). From \textbf{Theorem 9}, we only have to check if the following holds true for all players with higher contribution ($\phi$) than player $i$:
\begin{equation}
    \delta \leq \frac{d_s \phi_{max}}{\phi_{max}-\phi_s}\hspace{2mm} \forall s \in \hspace{1mm} \{s: \phi_s > \phi_i\}
\end{equation}
In the given example, assume that we would like to marginally change player 4's contribution according to expression (17) based on characteristic function $v_1(.)$. Then, we simply check if the following three inequalities (for player 5,6 and 7)  hold: \begin{equation}
    \begin{split}
        & \delta \leq \frac{d_5\phi_{max}}{\phi_{max}-\phi_5} \\
        & \delta \leq \frac{d_6\phi_{max}}{\phi_{max}-\phi_6} \\
        & \delta \leq \frac{d_7\phi_{max}}{\phi_{max}-\phi_7}
    \end{split}
\end{equation}
Omitting the arithmetic derivation, the three inequalities can be summarised to the inequality condition: $0 < \delta \leq 249.4$ for the given synthetic dataset we created and $v_1(.)$. Hence, the optimal outcome is still achievable if and only if the marginal contribution of player 4 is increased not more than 249.9. We set $\delta$ to 200 and 400 and display the outcomes in Figure 8 below.

\begin{figure}[h]
\includegraphics[scale=0.5]{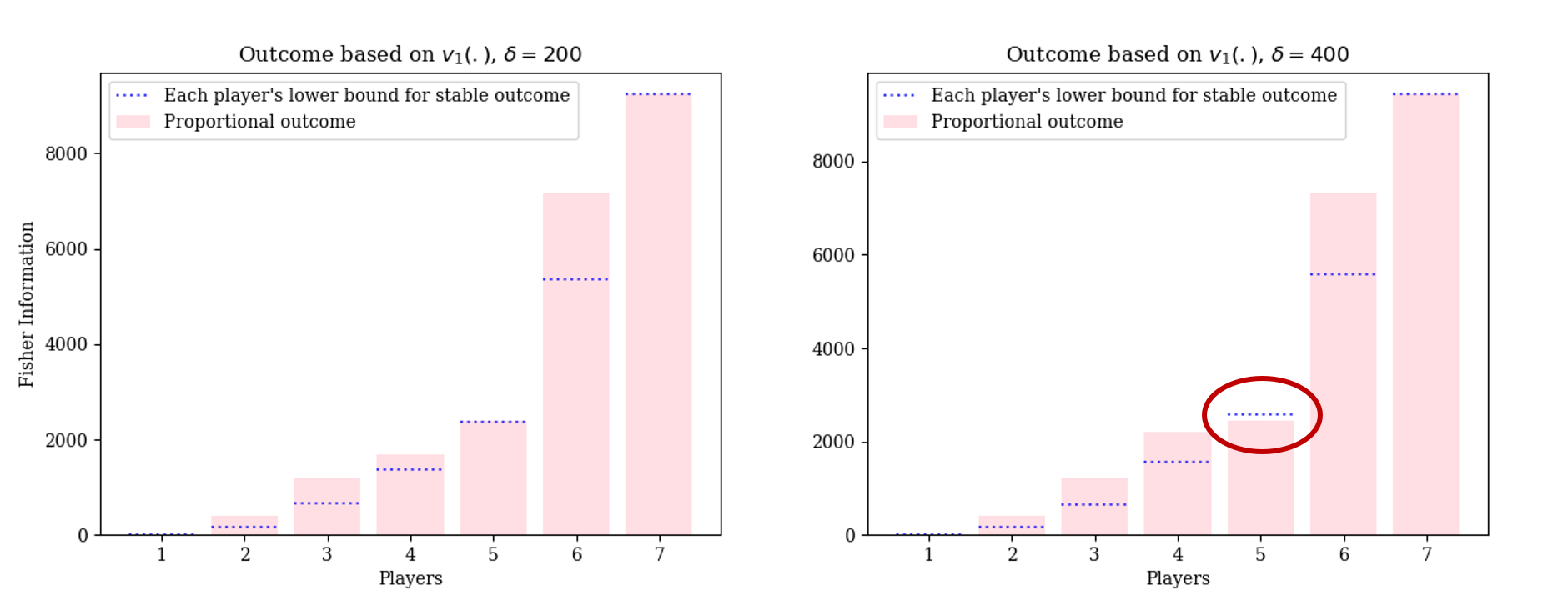}
\centering
\caption{Outcomes when player 4's contribution is increased by $\delta = 200 $ and $\delta =  400$}
\end{figure}

As expected, since $\delta=200$ satisfies the inequality condition, the outcome is still optimal on the left, whereas when $\delta = 400$ violates the inequality condition, the outcome is no longer optimal (the \textcolor{red}{red circle} highlights that player 5's proportional payoff is not stable, and hence an optimal outcome is not achievable) on the right.

\subsection{Extensions on other machine learning models}
\subsubsection{Logistic Regression for classification problems}
Similar to Linear Regression, the \textit{Fisher Information} for a Logistic Regression model is well defined and indicative of the variance of the estimated parameters as well. The derivation \cite{13} is well known and we summarise the results below.

Let $\theta = (\theta_1,\dots,\theta_p)^T\in \mathcal{R}^p$ be the model parameters. For the dataset $D=\{(x_t,y_t)\}$ with $n$ data points, the likelihood of a Logistic Regressionmodel is as follows:
\begin{equation}
    f(y|x,\theta) = P(y+1|x,\theta)) = \frac{1}{1+e^{-{\theta^Tx}}}
\end{equation}
Then the \textit{Fisher Information} (in matrix form since we have more than one parameter now) has entries:
\begin{equation}
    \mathcal{I}(\theta) = X^TWX
\end{equation}
where we have $X$ as the data set in matrix form and defined the $n \times n$ diagonal matrix:
\begin{equation}
    W = \text{diag}\left(\frac{e^{\sum^p_{j=0}\theta_j x_{1j}}}{(1+e^{\sum^p_{j=0}\theta_j x_{1j}})^2},\dots,\frac{e^{\sum^p_{j=0}\theta_j x_{nj}}}{(1+e^{\sum^p_{j=0}\theta_j x_{nj}})^2}\right)
\end{equation}
We can use the sum of diagonals, $(X^TWX)_{jj}$, of the \textit{Fisher Information} matrix as the characteristic function associated to a set of data, since large diagonal entries correspond to smaller uncertainty about one of the parameters using the Cramer Rao Bound (\textbf{Definition 11})
\subsubsection{Neural Networks}
For neural networks with softmax or sigmoid output layers (whose output can be viewed as a likelihood), the \textit{Fisher Information} and \textit{Mutual Information} are also well defined for each parameter within the networks. However, due to large number of parameters in modern neural networks, it may be impractical to perform large matrix arithmetic for these measures.

Agencies can resolve this by adopting a pre-trained neural network and only train the parameters on the last few layers, greatly reducing the number of parameters in question (of course, this requires the pretrained model to be trained on a similar task).
\subsection{Extension on real life data}
We extend the data sharing process to a real life data set consisting of breast cancer tumour diagnosis \cite{16}. Such data can be held in small sets separately by a few hospitals, and our data sharing process investigates their contribution and rewards when they work together. The data is labeled with a binary value indicating whether a tumour is malignant based on various attributes of its appearance such as perimeter and smoothness. We first train a neural network using a small subset of data to simulate a pre-trained neural network with two hidden layers of 6 hidden units each. Our data sharing process learns 4 parameters (3 parameters for each unit in the exposed layer and 1 offset parameter) associated with the last layer in a Logistic Regression setting. Figure 9 highlights the neural network structure.

\begin{figure}[h]
\includegraphics[scale=0.9]{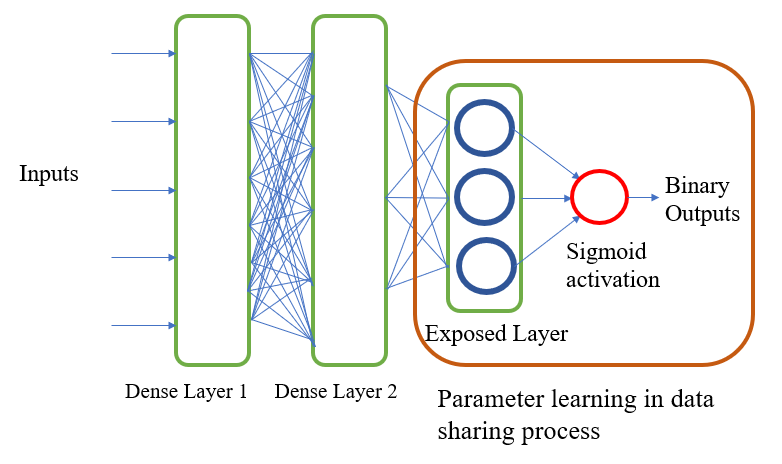}
\centering
\caption{Pre-trained neural network with the data sharing process done over the last layer}
\end{figure}
Hence, although the breast cancer data set has 10 features, the neural network transforms it into a 3 dimensional data set in the first 2 layers. This transformed data is then trained over a Logistic Regression model in the final layer. To further demonstrate that some data points in the 3 dimensional space is more useful than others, we use the trace of the \textit{Fisher Information} matrix as our characteristic function in a Logistic Regression setting with $\mathcal{I}(\theta) = X^TWX$ defined in Equation (21).

The expression tells us that $\mathcal{I}(\theta)$ has large diagonal entries when $W$, a diagonal matrix, has large diagonal entries. Equation (22) further implies that each diagonal entry corresponds to $\hat{p}(1-\hat{p})$ of a contributed data point, where $\hat{p}$ is the output of the Logistic Regression model for that data point. Thus, the diagonal entry is large when the corresponding data point contributed lies close to the decision boundary of the learnt Logistic Regression model.

As such, if the \textit{Fisher Information} is used as the characteristic function used to gauge the value of data, data points which lie close to the learnt decision boundary will have higher value (Figure 10). To demonstrate this, we assume 3 hospitals each with the same number of data points. They hold data points in 3 different regions: 
\begin{itemize}
    \item \textbf{Hospital A} holds only data points which lies near the boundary ($ 0.25 < \hat{p} < 0.75 $)
    \item \textbf{Hospital B} holds data points which can be moderately far from the boundary ($ 0.1 < \hat{p} < 0.9 $)
    \item \textbf{Hospital C} holds data points which can be very far from the boundary ($ 0.05 < \hat{p} < 0.95 $)
\end{itemize}
\begin{figure}[h]
\includegraphics[scale=0.8]{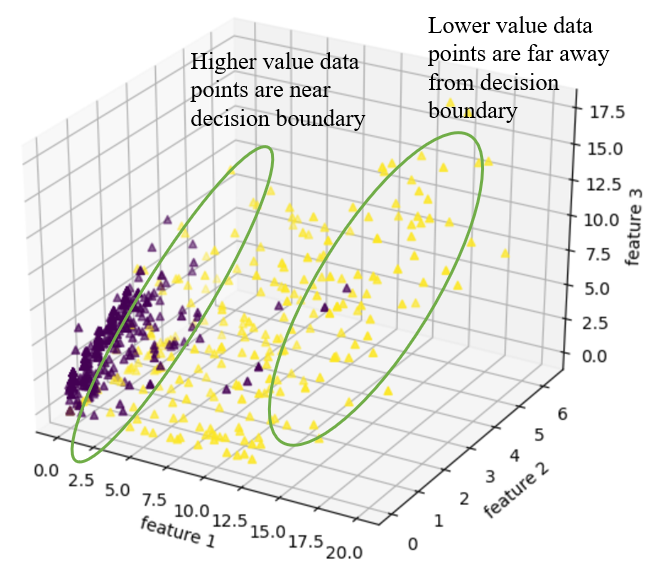}
\centering
\caption{Varying importance (based on \textit{Fisher Information}) of real life data after neural network transformation}
\end{figure}

\begin{figure}[h]
\includegraphics[scale=0.6]{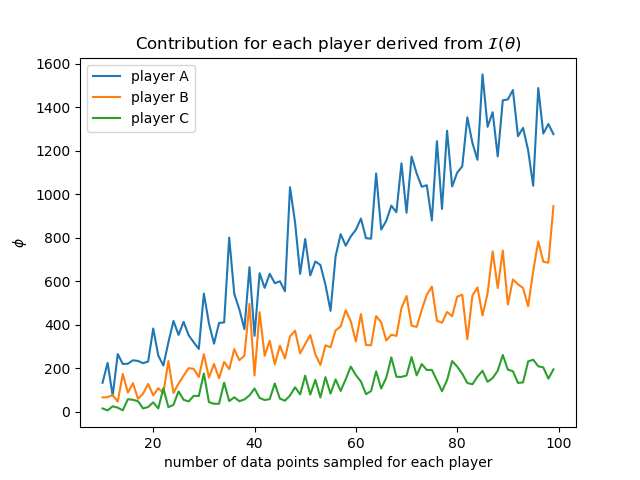}
\centering
\caption{\textit{Shapley value} of each hospital based on \textit{Fisher Information}}
\end{figure}
We see that randomly sampling data points for each player at different regions of the feature space does indeed give different level of contribution in the data sharing process. To calculate whether the outcome is optimal in terms of the model awarded to each hospital, one applies \textbf{Theorem 6} based on the given \textit{Shapley value} of each hospital. In this case, we assume the data set is divided amongst the three hospital (around 200 data points each) such that each holds equal number of data points but of different quality. It turns out, fortunately, the outcome is optimal (proportional and stable):

\begin{center}
 \begin{tabular}{||c c c c||} 
 \hline
    & \textbf{Hospital A} & \textbf{Hospital B} & \textbf{Hospital C}   \\ [0.5ex] 
 \hline
 $\phi$ (contribution) & 317 & 1369 & 2801  \\
 \hline
 Reward & 508 & 2194 & 4488 \\ 
 \hline
\end{tabular}
  \\[1.5mm] 
  Table 2: \textit{Shapley value} based on two different valuation function $v_1$ and $v_2$
\end{center}
We observe that each hospital receives a better model than if it had worked alone. This implies the data sharing process creates higher quality model for each participating hospital and is desirable for society at large.
\section{Future Work}
\subsection{Choosing reasonable characteristic functions}
Choosing a sound characteristic function to measure players' contribution in a data-sharing process is crucial in calculating contribution and generating fair reward for each player. In corporate practice, players and agencies may find it inconvenient to use complicated measures (such as \textit{Fisher Information} and \textit{Mutual Information}) due to the lack of understanding of their theoretical interpretations. This may require some domain experts to be involved in such data sharing projects.

From a theoretical stand, it would be desirable to categorise characteristic functions into classes based on certain function properties (e.g monotonicity,  submodularity,  superadditivity). We can then study what outcomes are possible for different classes of functions. This is useful because when faced with a new characteristic function, we can try to see if such a function falls into any of these function classes and immediately reach some known conclusions about the outcome.

\subsection{Desirable suboptimal outcomes}
This paper offered a few suboptimal outcomes when optimality cannot be achieved in Section 7. However, it is not clear which outcomes are actually desired by participating data providers in real life. One may need to rely on some empirical studies on the behaviour of players in terms of preference and rationality.
\section{Conclusion}
This paper aimed to provide investigative insights towards the data sharing process with regards to players' contribution evaluation and outcomes, which have not been studied much previously.

In this paper, we formulated the data sharing process as a new class of cooperative game with new definition of stability and fairness. Following which, we provided theoretical analyses on the various optimal and suboptimal outcomes of such games in relation to any arbitrary characteristic function and explained how these outcomes may be useful to a central mediator in the data sharing process in terms of explaining the behaviours of participating players. Finally, we performed sensitivity analysis to study how sensitive an optimal outcome is towards marginal changes.

Through experiments on synthetic real life datasets, we demonstrated the theoretical results introduced in the paper, offered intuitive meaning behind them and explained how these results can be extended towards different machine learning models.
\newpage

\addcontentsline{toc}{section}{References}
\nocite{*}
\bibliographystyle{apalike}
\bibliography{ref}

All code in the experiments can be found at: \\ \url{https://github.com/chenzhiliang94/multi-agent-data-sharing}
\end{document}